\documentclass[letterpaper, 10 pt, conference]{ieeeconf}  % Comment this line out if you need a4paper

\IEEEoverridecommandlockouts                              % This command is only needed if 
\usepackage{graphicx} % for pdf, bitmapped graphics files
\usepackage{amsmath} % assumes amsmath package installed
\usepackage{amssymb}  % assumes amsmath package installed
\usepackage{algorithm}
\usepackage{algorithmicx}
\usepackage[noend]{algpseudocode}
\usepackage{subcaption}
\usepackage{url}
\newtheorem{theorem}{Theorem}
\algnewcommand{\LeftComment}[1]{\Statex \(\triangleright\) #1}

\newcommand{\statespace}{\mathbb{S}}
\newcommand{\start}{s_\mathsf{start}}
\newcommand{\goal}{s_\mathsf{goal}}
\newcommand{\goalset}{S_\mathsf{goal}}
\newcommand{\successors}{\mathcal{S}}
\newcommand\dup{\mathsf{dup}}
\newcommand\dupe{\mathsf{dup}_\mathcal{E}}
\newcommand\dups{\mathsf{dup}_\mathcal{S}}
\newcommand\epsilonmax{\epsilon_\mathsf{max}}
\newcommand\open{\mathsf{OPEN}}
\newcommand\closed{\mathsf{CLOSED}}
\newcommand\dist{\mathsf{dist}}
\newcommand\parent{s_\mathsf{parent}}
\newcommand\dNN{d_\mathsf{NN}}
\newcommand\hashtable{\textsc{HashSubtree}}
\newcommand\labeler{\textsc{Subtree}}
\newcommand\weiapproach{\textsc{Penalty}}
\newcommand\wastar{\textsc{WA*}}
\newcommand\rrt{\textsc{RRT}}
\newcommand\rrtstar{\textsc{RRT*}}

\title{\LARGE \bf
Improved Soft Duplicate Detection in Search-Based Motion Planning
}

\author{Nader Maray$^{1}$, Anirudh Vemula$^{2}$, and Maxim Likhachev$^{2}$%
\thanks{$^{1}$ Department of Computer Science, Texas State University {\tt \small
    fmn9@txstate.edu}}%
\thanks{$^{2}$ Robotics Institute, Carnegie Mellon University {\tt \small
    vemula@cmu.edu maxim@cs.cmu.edu}}%
}

\begin{document}

\maketitle
\thispagestyle{empty}
\pagestyle{empty}

%%%%%%%%%%%%%%%%%%%%%%%%%%%%%%%%%%%%%%%%%%%%%%%%%%%%%%%%%%%%%%%%%%%%%%%%%%%%%%%%
\begin{abstract}

Search-based techniques have shown great success in motion planning
problems such as robotic navigation by discretizing 
the state space and precomputing motion primitives. However in
domains with complex dynamic constraints, constructing motion
primitives in a discretized state space is non-trivial. This requires
operating in continuous space which can be challenging for
search-based planners as they can get stuck in local minima regions.
Previous work~\cite{DBLP:conf/iros/DuKSL19} on planning in continuous
spaces introduced \textit{soft 
  duplicate detection} which requires search to compute the duplicity
of a state with respect to previously seen states to avoid exploring
states that are likely to be duplicates, especially in local minima regions.
They propose a simple metric utilizing
the euclidean distance between states, and
proximity to obstacles to compute the duplicity. In this paper, we
improve upon this metric by introducing a kinodynamically
informed metric, \textit{subtree overlap}, between two states as the similarity
between their successors that can be reached within a fixed time
horizon using kinodynamic motion primitives. This
captures the intuition that, due to robot dynamics, duplicate states can be
far in euclidean distance and result in very similar successor states,
while non-duplicate states can be close and
result in widely different successors.
Our approach computes the new metric offline for a given robot dynamics, and stores
the subtree overlap value for all possible
relative state configurations.
During search, the planner uses these precomputed values to speed up duplicity
computation, and achieves fast planning times in continuous spaces in addition
to completeness and sub-optimality guarantees.
Empirically, we show that our improved metric for soft
duplicity detection in search-based planning outperforms previous
approaches in terms of planning time, by a factor of
$1.5$ to $2$$\times$ on $3$D and $5$D planning
domains with highly constrained dynamics.
% @avemula: Add speedups obtained and concise details on experimental
% domains
% @avemula: Still lacking the definition of what duplicity is, but do
% we need it?

\end{abstract}

\section{Introduction and Related Work}
Planning motion for robots such as
manipulators~\cite{DBLP:conf/icra/BerensonSFK09}, unmanned aerial
vehicles~\cite{DBLP:journals/ras/AllenP19} and
humanoids~\cite{DBLP:conf/wafr/HauserBHL06} with complex kinodynamic 
constraints is challenging as it 
requires us to compute trajectories that are both collision-free and feasible to
execute on the robot. 
The traditional approach to kinodynamic planning using search-based planners is
to discretize the continuous state space into cells, and the search traverses
only through the centers of the cells~\cite{DBLP:conf/iros/LiuAMK17}. To account
for the kinodynamic nature of the planning problem, these
approaches~\cite{DBLP:conf/iros/PivtoraikoK11,
  DBLP:conf/icra/CohenCL10} make use of motion
primitives which are precomputed
actions that the robot can take at any state. Due to the discretization,
we require that the motion primitives be able to connect from one cell
center to another cell center. However, this requires solving a two-point boundary value
problem~\cite{DBLP:journals/jfr/PivtoraikoKK09} which may be infeasible to solve
for domains with highly constrained
dynamics~\cite{DBLP:books/cu/L2006}. A typical solution is to 
discretize the state space at a higher resolution which can blow up the size of
the search space making search computationally very expensive~\cite{DBLP:journals/iandc/Bellman58}.

Alternatively, we can use sampling-based approaches such as
\rrt{}~\cite{DBLP:journals/ijrr/LaValleK01} and
\rrtstar{}~\cite{DBLP:journals/ijrr/KaramanF11} 
that have been popular in kinodynamic planning.
These approaches
directly plan in the continuous space by randomly sampling controls or
motion primitives, 
and extending states~\cite{DBLP:journals/ijrr/LaValleK01,
  DBLP:journals/arobots/SakcakBFP19} until we reach the goal
state.
% The major disadvantage 
% of these approaches is that they do not provide any
% guarantees on sub-optimality 
% of the solution and lack deterministic behavior~\cite{DBLP:books/cu/L2006}, unlike search-based
% methods.
The major disadvantage of these approaches is that typically the
solutions they generate can be quite poor in quality and lack
consistency in solution (similar solutions for similar planning
queries) due to randomness, unlike search-based methods.
Furthermore, in motion planning domains 
with narrow passageways, sampling-based
approaches can take a long time to find a solution as the probability
of randomly sampling a state within the narrow passageway is very
low~\cite{DBLP:conf/ro-man/MainpriceRTS20}. 

Naively, one could use search-based planning directly in the
continuous space. The challenge here is to ensure that the search does not
unnecessarily expand similar states, which happens often when the
heuristic used is not informative and 
does not account for the kinodynamic constraints explicitly. In such cases, the
search can end up in a ``local minima'' region, and expand a large number of
nearly identical states before exiting the region~\cite{DBLP:conf/aaai/LikhachevS08}.
Approaches that aim to detect when the heuristic is ``stagnant''~\cite{DBLP:conf/ijcai/IslamSL18} or
avoid local minima regions in the state space~\cite{DBLP:conf/aips/VatsNL17} require either
extensive domain knowledge or user input to achieve this. In order to
avoid expanding similar 
states, approaches such as~\cite{DBLP:journals/algorithmica/BarraquandL93,
  DBLP:conf/socs/GonzalezL11} group states into equivalence classes that can be
used for duplicate detection~\cite{DBLP:conf/ijcai/DowK09} in discrete state
spaces.

A recent work~\cite{DBLP:conf/iros/DuKSL19} extends similar ideas to continuous
state spaces by introducing \textit{soft duplicate detection}.
They define a duplicity function that assigns a value to each state during
search based on how likely that state can contribute to the search finding a
solution. The assigned duplicity is then used to penalize states that are
similar to previously seen states by inflating their heuristic within a weighted
A* framework~\cite{DBLP:journals/ai/Pohl70}. Unlike past duplicate detection
works, this approach does not prune away duplicate states resulting in maintaining
completeness guarantees while achieving fast planning times. Their approach is
shown to outperform other duplicate detection approaches such
as~\cite{DBLP:journals/algorithmica/BarraquandL93,
  DBLP:conf/socs/GonzalezL11}.

Our proposed approach builds upon this work by improving the duplicity function
used, to incorporate a more kinodynamically informed notion of when a state is likely to
contribute to the search finding a solution. We achieve this by using
a novel metric, \textit{subtree overlap}, between two states
which is a similarity metric between their successors that can be
reached using kinodynamic motion primitives within a fixed time
horizon. We show that computing this new metric during search can be
computationally very expensive and hence, we precompute the metric
offline for all possible relative state 
configurations and store these values. 
During search, we use these precomputed subtree overlap values in
duplicity computation to obtain $1.5$ to $2$$\times$ improvement in planning time
over previous approaches, including~\cite{DBLP:conf/iros/DuKSL19}, in
continuous $3$D and $5$D planning domains with highly constrained dynamics. 
In addition to faster planning times, we also retain the completeness and
sub-optimality bound guarantees of search-based methods.

\section{Problem Setup and Background}
In this section, we will describe the problem setup and previous work
on soft duplicate
detection~\cite{DBLP:conf/iros/DuKSL19} as
our approach builds on it. We are given a robot with state space
$\statespace$ whose dynamics are constrained, and the objective is to
plan a kinodynamically feasible and collision-free path from a start
state $\start$ to any state in a goal region specified by
$\goalset \subset \statespace$. This is formulated as a search problem by constructing a
lattice graph using a set of motion primitives, which are
dynamically feasible actions, at any
state~\cite{DBLP:journals/jfr/PivtoraikoKK09}. These motion primitives
define a set of successors for any state $s$ given by the set
$\successors_{1}(s)$ comprising of states that are dynamically feasible to
reach from $s$. We also
have a cost function $c: \statespace \times \statespace \rightarrow
\mathbb{R}^+ \cup \{0\}$ which assigns a cost $c(s, s')$ for any
motion primitive taking the robot from state $s$ to $s'$. Finally, we
assume that we are given access to an admissible heuristic function
$h: \statespace \rightarrow \mathbb{R}^+ \cup \{0\}$ which is an
underestimate of the optimal path cost from any state to a goal state.

The soft duplicate detection framework, introduced
in~\cite{DBLP:conf/iros/DuKSL19}, requires access to a duplicity
function $\dup: \statespace \times \statespace \rightarrow [0, 1]$
which for any two states $s, s'$ corresponds to the likelihood that
they are duplicates of each other. More precisely, $\dup(s, s')$
captures the (inverse) likelihood that $s$ will contribute to
computing a path to $\goalset$ given that the search has already
explored $s'$. This can be generalized (with a slight
abuse of notation) to compute duplicity of a state with respect to a
set of states $U$ as $\dup(s, U) = \max_{s' \in U} \dup(s,
s')$. This framework is used within Weighted
A*~\cite{DBLP:journals/ai/Pohl70} search\footnote{Recall that weighted A*
expands states in order of their $f$-value given by $f(s) = g(s) + \epsilon_0 h(s)$
where $g(s)$ is the cost-to-come for $s$ from $\start$, and
$\epsilon_0$ is the inflation factor} by using a state-dependent
heuristic inflation factor
given by,
$\epsilon(s) = \max(\epsilonmax \cdot \dup(s, U), \epsilon_0)$
where $\epsilon_0$ and $\epsilonmax$ are constants such that
$\epsilonmax > \epsilon_0 > 1$. The set $U$ consists
of the states in the priority queues $\open$ (containing states that
the search may expand,) and $\closed$ (containing states that search
has already expanded,) i.e. $U = \open \cup \closed$. For completeness, we
present Weighted A* with soft duplicate detection using $\dup(s, U)$
from~\cite{DBLP:conf/iros/DuKSL19} in Algorithm~\ref{alg:wastar}.
% Our approach
% uses the same search algorithm but with a different duplicity function.
Note
that states with higher duplicity have a higher inflation factor leading to
giving them lower priority in $\open$ maintained by weighted A*, and vice versa.
% If the underlying heuristic is admissible, and since $\dup(s, s') \in [0, 1]$,
% then Algorithm~\ref{alg:wastar} is guaranteed to produce results with a bounded
% sub-optimality factor of $\epsilonmax$.

\begin{algorithm}[t]
  \small
\caption{Weighted A* with Soft Duplicate Detection~\cite{DBLP:conf/iros/DuKSL19}}
 \label{alg:wastar}
\begin{algorithmic}[1]
  \State {\bfseries Input:} Duplicity Function $\dup$, constants $\epsilon_{0}, \epsilonmax$
  \State Set $g(\start) = 0$ \Comment{$g$-value for other states set to $\infty$}
  \State Set $\closed = \phi$, and $\open = \{\start\}$
  \While {$\goal \notin \closed$ and $\open \neq \phi$}
  \State Set $U = \open \cup \closed$
  \State Pop $s$ with the smallest $f$-value from $\open$
  \State $\closed = \closed \cup \{s\}$
  \For {$s' \in \successors_{1}(s)$}
  \If {$s' \notin U$}
  \State $\epsilon(s') = \max(\epsilonmax \cdot \dup(s', U), \epsilon_{0})$
  \EndIf
  \If {$g(s') > g(s) + c(s, s')$}
  \State $g(s') = g(s) + c(s, s')$
  \State $f(s') = g(s') + \epsilon(s')\cdot h(s')$
  \State $\open = \open \cup \{s'\}$ \Comment{Key is $f$-value}
  \EndIf
  \EndFor
  \EndWhile
  \If {$\goal \in \closed$}
  \State \Return {Path found from $\start$ to $\goal$}
  \EndIf
  \State \Return {No path found}
\end{algorithmic}
\end{algorithm}

Wei et. al.~\cite{DBLP:conf/iros/DuKSL19} propose using a duplicity
function $\dupe$ that is defined as follows,
\begin{equation}
  \label{eq:3}
  \dupe(s, U) = 1 - \frac{\dNN(s, U)}{R \cdot \gamma(\parent)}
\end{equation}
where $\dNN(s, U) = \min_{s' \in U} \dist(s, s')$, $\dist$ is a euclidean
distance metric defined on 
$\statespace$, $R$ is a distance normalization constant, and
$\gamma(\parent)$ is the valid successor ratio that computes the
proportion of successor states of the parent of $s$ that are not in
collision with obstacles. Specifically, for any state $s$ we have
\begin{equation}
  \label{eq:1}
  \gamma(s) = \frac{\text{\# collision-free states in $\successors_{1}(s)$ }}{\text{\#
      states in $\successors_{1}(s)$}}
\end{equation}
Intuitively, the duplicity function in equation~\eqref{eq:3} assigns higher
duplicity to states that are close to each other, and vice versa. It
also assigns lower duplicity for states that are in the vicinity of
obstacles to ensure that the search does not penalize states within
tight workspaces such as narrow passageways.

\section{Approach}
In this section, we will present our novel metric for duplicity
computation in the soft duplicate detection framework. We will start
with introducing subtree overlap, a metric that accounts for
kinodynamic constraints explicitly. Subsequently, we
present simple examples
to motivate why we can expect subtree overlap to be a better
indication of duplicity when compared to using euclidean distance
alone. Finally, we present our
approach where we precompute subtree overlap values, given robot dynamics, for
all possible relative state configurations, and use the stored precomputed values
during search to speed up planning.

\begin{figure*}[t]
  \centering
  \begin{subfigure}{0.2\linewidth}
  \includegraphics[width=\linewidth]{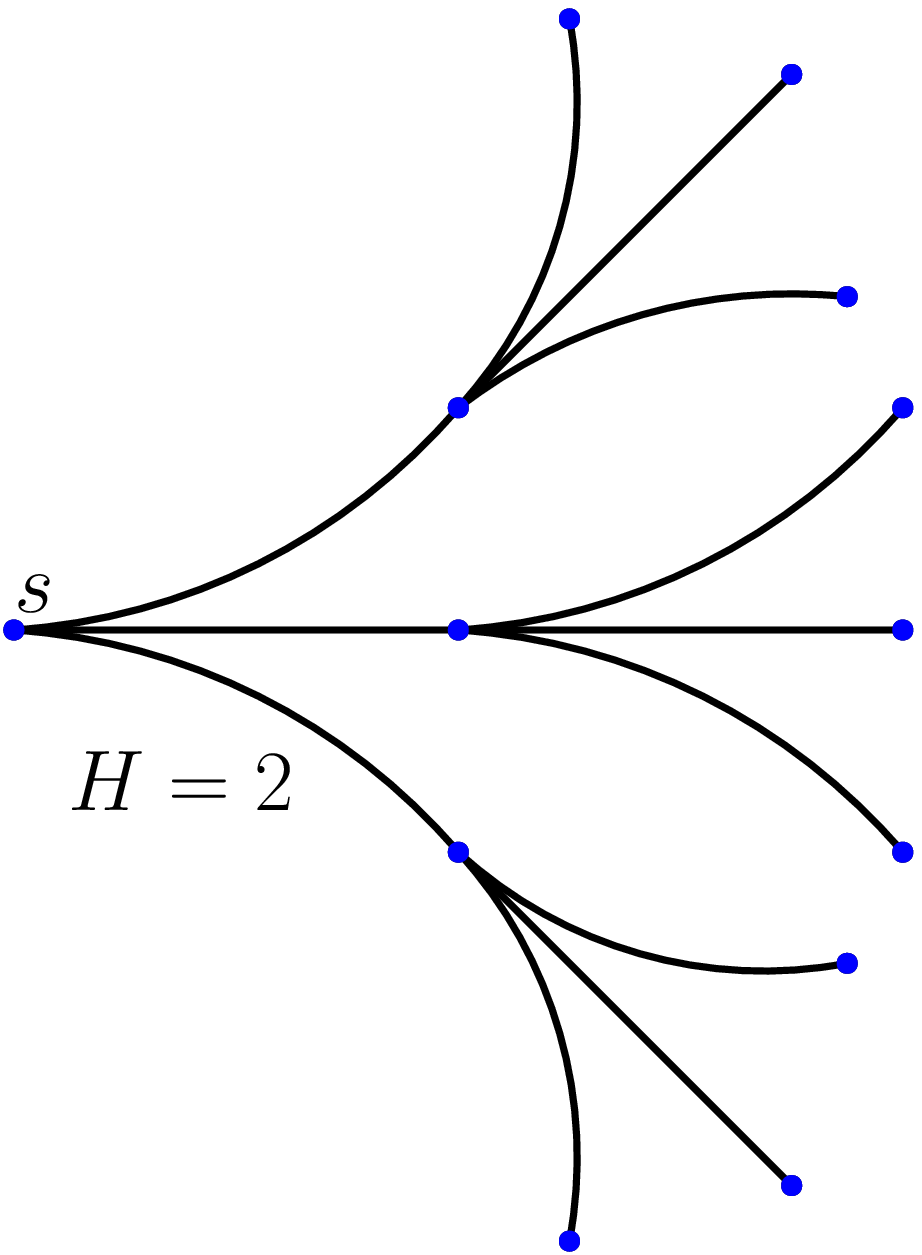}  
  \end{subfigure}\hfill
  \begin{subfigure}{0.4\linewidth}
    \includegraphics[width=\linewidth]{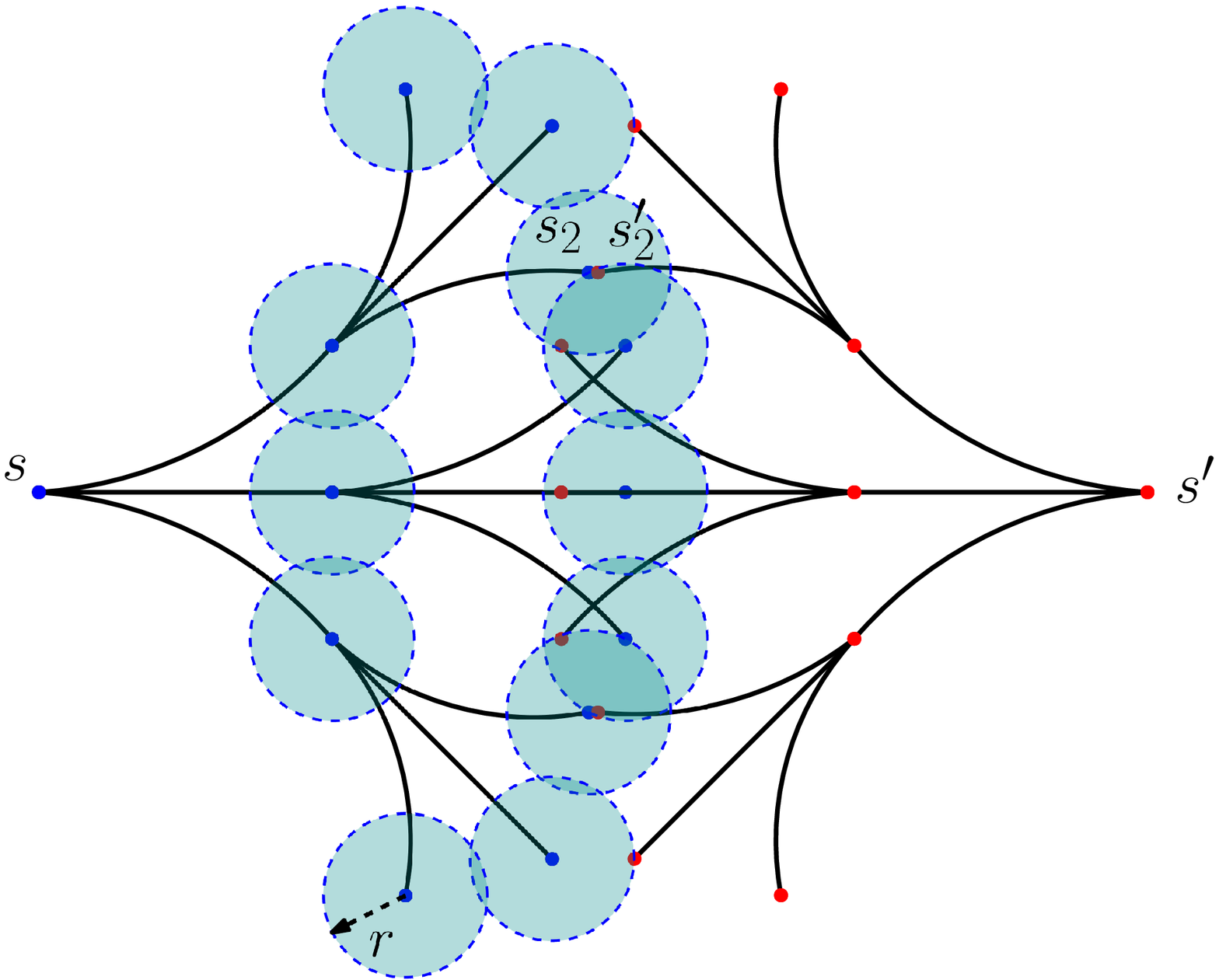}
  \end{subfigure}\hfill
  \begin{subfigure}{0.2\linewidth}
    \includegraphics[width=\linewidth]{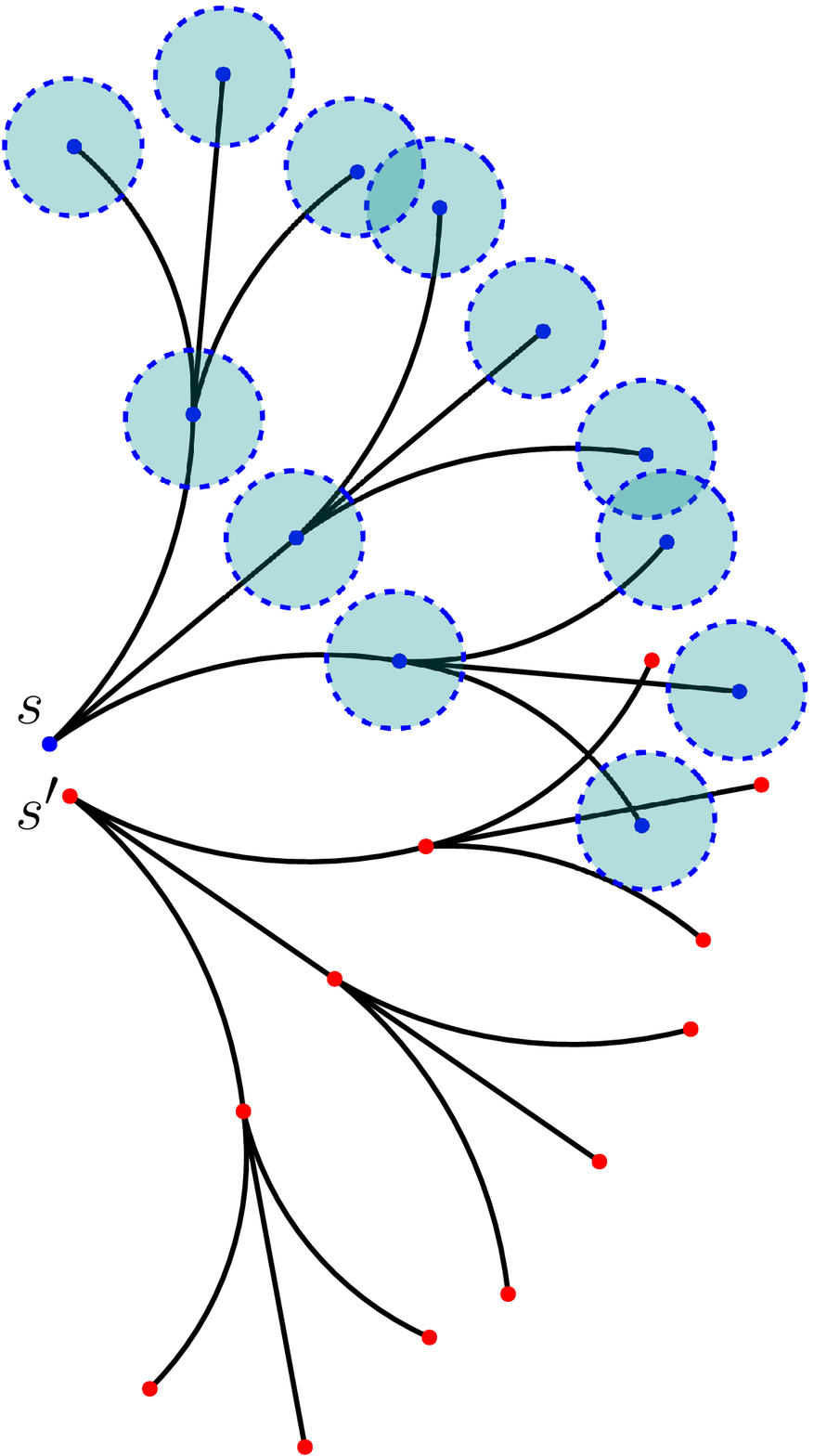}
  \end{subfigure}\hfill
  \caption{(left) Example subtree $\successors_2(s)$ constructed at a state $s$ with a
    depth $H=2$. (middle) Two states $s, s'$ that are far from each
    other but have significant overlap in subtrees resulting in
    $\eta_2(s, s') = \frac{7}{12}$. The circles have a radius $r$ and
    are centered on each state in $\successors_2(s)$. (right) Two states
    $s, s'$ that are close to each 
    other but have no overlap in subtrees resulting in $\eta_2(s, s')
    = 0$}
  \label{fig:subtree}
  \vspace{-0.2cm}
\end{figure*}

\subsection{Subtree Overlap}
\label{sec:subtree-overlap}

Given a set of states $U$ that have already been
encountered by the search, a new state $s$ is useful if it
helps the search explore a 
new region of the state space that would not have been explored by expanding
states in $U$. This observation allows us to come up with a
kinodynamically informed metric for
duplicate detection.
Let $\successors_{H}(s)$ denote the successors of state $s$ that
are dynamically feasible to reach through the execution of $H$ (or less) motion
primitives in sequence.
This set can be represented using a subtree, as shown in
Figure~\ref{fig:subtree}(left), rooted at $s$ and with a depth $H$. Observe that this
subtree captures the region of state
space that will be potentially explored within $H$ motion primitives when we
expand the state $s$ during search.

For a new state $s$ which is being added to $\open$ whose duplicity needs to be
evaluated (see Algorithm~\ref{alg:wastar},) we first construct the
subtree $\successors_{H}(s)$ for some fixed $H > 0$. To understand if expanding
$s$ allows us to explore a new region of the state space, we can consider any state
$ s' \in U = \open \cup \closed$ and construct the subtree $\successors_{H}(s')$
that contains the region of state space that can be explored by expanding $s'$. Comparing
$\successors_{H}(s')$ with $\successors_{H}(s)$ allows us to evaluate the
relative utility of $s$ in computing a path to a goal in $\goalset$ given that
the search has already seen state $s'$. To capture this quantity precisely,
we define the metric \textit{subtree overlap} of a state $s$ with another
state $s'$ denoted by $\eta_{H}(s, s')$.

The subtree overlap $\eta_{H}(s, s')$ is computed as the proportion of states
in subtree $\successors_{H}(s)$ that ``overlap'' with any state in the subtree
$\successors_{H}(s')$. We consider two states to be
overlapping if the euclidean distance between them is less than a small constant
$r$, and they lie along the same depth in their respective subtrees. For example,
consider state $s_{2}$ that is at depth $2$ in the subtree
$\successors_{H}(s)$ as shown in Figure~\ref{fig:subtree}(middle). If the subtree
$\successors_{H}(s')$ contains a state
$s'_{2}$ that is also at depth $2$ and is less than $r$ distance away from
$s_{2}$, then we consider $s_{2}$ to be overlapping with $s'_{2}$. Thus we have,
\begin{equation}
  \label{eq:4}
  \eta_{H}(s, s') = \frac{\text{\# states in $\successors_{H}(s)$ that
    overlap with $\successors_{H}(s')$}}{\text{\# states in $\successors_{{H}}(s)$}}
\end{equation}

Intuitively, when we have a high $\eta_{H}(s, s')$ (close to $1$) then the state
$s$ is likely to be a duplicate of $s'$ as they lead to very similar
successors. One such example is shown in Figure~\ref{fig:subtree}
(middle) where we have two states $s, s'$ that are far in euclidean
distance but result in similar successors and thus, have a high
$\eta_H(s, s')$. Using euclidean distance alone would lead us to
incorrectly conclude that $s, s'$ are not duplicates.

On the other hand, if we have a low $\eta_{H}(s, s')$ (close to $0$) then both
states lead to widely different successors, hence are not likely to be
duplicates of each other. An example of this is shown in
Figure~\ref{fig:subtree} (right) where we have two states $s, s'$ that
are close to each other, but have minimal overlap in successors and
thus, have a low $\eta_H(s, s')$. These states would be incorrectly
considered duplicates if we solely used euclidean distance.

% \subsection{Motivating Examples}
% \label{sec:examples}

% Subtree $\successors_{H}(s)$ captures the region of state space that is
% \textit{dynamically} feasible to reach from state $s$ using motion primitives.
% This allows us to reason about robot dynamics explicitly when computing the
% duplicity of a state with respect to another state. In domains with highly
% constrained dynamics, we could have two states $s, s'$ that are close to each
% other in euclidean space but result in widely different successors due to the
% the robot dynamics. One such example is shown in
% Figure~\ref{fig:subtree}(right), where we have two states $s$ and $s'$ which are
% close to each other, but their respective subtrees have minimal
% overlap. Duplicity functions, such as the 
% one proposed in~\cite{DBLP:conf/iros/DuKSL19}, that only accounts for euclidean
% distance will incorrectly assign a high duplicity to states that are
% close, and might
% result in the search taking a long time to find a path to the goal.

% On the other hand, we could also have a pair of states $s$ and $s'$, as shown in
% the example in Figure~\ref{fig:subtree}(middle), which are far in euclidean space
% but result in very similar successors as captured by our subtree overlap metric.
% Once again, simply relying on euclidean distance alone can lead to
% assigning a low duplicity for 
% the pair $(s, s')$ which can result in the search wasting expansions and taking
% a longer time to compute a path to the goal.

\subsection{Duplicate Detection using Subtree Overlap}
\label{sec:dupl-detect-using}

In this section, we will describe how we use the subtree overlap metric
presented in Section~\ref{sec:subtree-overlap} to compute the duplicity. Given a
decision boundary parameter $c \in [0, 1]$ and a distance normalization constant $R>>r$ (similar to
equation~\eqref{eq:3}) we have,
\begin{equation}
  \label{eq:5}
  \dups(s, s') = 1 - \frac{d(s, s')(1 + c - \eta_{H}(s, s'))}{R \cdot \gamma(\parent)}
\end{equation}
where $d(s, s')$ is the euclidean distance between $s$ and $s'$, and $\gamma$ is
the valid successor ratio as defined in equation~\eqref{eq:1}. We
chose $R$ such that $\dups(s, s') \in [0, 1]$.

To use the proposed duplicity function $\dups$ in the framework of
Algorithm~\ref{alg:wastar} we need to define how we compute $\dups(s, U)$ for a
set of states $U$. If we naively construct the subtree for all states
$s'$ in $U = \open \cup \closed$, it can be very expensive computationally as
constructing the subtree involves expanding several states and querying all of
their successors. However, computing subtree overlap, with a fixed depth $H$,
between states that are very far from each other can be unnecessary as they
would most likely not have any overlap in successors. Thus, to reduce our
computation budget, we only compute the subtree overlap with states in $U$ that
are within a euclidean distance of $R$. We obtain the set of states $U_{R}$
that are within a distance of $R$ by maintaining all states in $U$ in a
kd-tree~\cite{DBLP:journals/cacm/Bentley75}. A similar trick was also used
in~\cite{DBLP:conf/iros/DuKSL19}. Thus we have,
\begin{equation}
  \label{eq:6}
  \dups(s, U) = \max_{s' \in U_{R}} \dups(s, s')
\end{equation}
Using the proposed duplicity function $\dups$, we retain the completeness and
sub-optimality bound guarantees of weighted A* as stated in the following theorem:
\begin{theorem}
  \textit{The planner, using Algorithm~\ref{alg:wastar} with the duplicity function $\dups$ defined
  in equations~\eqref{eq:5} and~\eqref{eq:6}, is guaranteed to find a
  solution, if one exists, and the cost of the solution is within $\epsilonmax$
  times the cost of the optimal solution.}
\end{theorem}
\begin{proof}
Follows from the completeness and sub-optimality bound proof of weighted
A*~\cite{DBLP:journals/ai/Pohl70} after observing that $\dups(s, U)$ always lies
in $[0, 1]$ which implies that the maximum inflation on the heuristic is bounded
by $\epsilonmax$.
\end{proof}
% Our approach uses the above duplicity function $\dups(s, U)$ in
% Algorithm~\ref{alg:wastar} to perform soft duplicate detection in search-based
% planning, and since $\dups(s, U) \in [0, 1]$ we
% are guaranteed to compute solutions that have a bounded sub-optimality factor of
% $\epsilonmax$.

\subsection{Interpreting the Proposed Duplicity Function}
\label{sec:interpr-new-dupl}

The proposed duplicity function $\dups$ as defined in equation~\eqref{eq:5} is similar to the
duplicity function $\dupe$ used in~\cite{DBLP:conf/iros/DuKSL19} (as presented in
equation~\eqref{eq:3}) with two major differences: the constant $c$ and the use
of subtree overlap metric $\eta_{H}(s, s')$. It is important to observe that the
constant $c$ plays a key role in how the subtree overlap metric
$\eta_{H}(s, s')$ affects the euclidean distance term $d(s, s')$ in equation~\eqref{eq:5}. If
$\eta_{H}(s, s') < c$, then
we \textit{scale up} the distance term $d(s, s')$, while if
$\eta_{H}(s, s') > c$ we \textit{scale down} the distance
term $d(s, s')$.

This allows us to reinterpret the subtree overlap metric as defining a new distance
metric $d_{\mathcal{S}}(s, s') = d(s, s')(1 + c - \eta_{H}(s, s'))$ where we
assign large values to pairs $(s, s')$ that are far away in euclidean
space and have low subtree overlaps. Similarly, we assign small values to pairs
$(s, s')$ that are close to each other in euclidean space and have large subtree
overlaps. Using subtree overlap in addition to euclidean distance allows us to
explicitly reason about robot dynamics, and avoid the pitfalls of using euclidean
distance alone as evidenced by our examples in Section~\ref{sec:subtree-overlap} and
our experimental results in Section~\ref{sec:experiments}.

% An observant reader would notice that constructing subtrees
% $\successors_H(s)$ requires expanding $s$ and its successors which we
% would like to avoid in the first place, especially when $s$ is a
% duplicate.
An observant reader could
argue that using $\dups$ allows us to lookahead and penalize states
earlier in search, while using $\dupe$ (as done
in~\cite{DBLP:conf/iros/DuKSL19}) instead would expand the states and
penalize overlapping
successors later in search. This raises the question of the usefulness of
$\dups$ over $\dupe$. While the reader's intuition is correct, the effect
on search progress using $\dups$ can be dramatic as it allows us to order
$\open$ more effectively for faster solution computation. Consider the example shown
in Figure~\ref{fig:subtree} (right) where $s, s'$ are close to each
other but have minimal subtree overlap. Using $\dupe$, $s$ would be
penalized and added to $\open$ with a very low priority. This delays
expanding $s$ until a much later stage in search, and can result in
long planning times if one of the successors of $s$ is crucial in
computing a path to the goal. However, using $\dups$ allows us to place
$s$ in $\open$ with a higher priority as it is not a duplicate, and
enable search to quickly expand $s$ and explore its successors leading
to the search finding a solution quickly.

\subsection{Precomputing Subtree Overlap Offline}
\label{sec:prec-subtr-overl}

In Algorithm~\ref{alg:wastar}, for every state $s$ that is about to be added to
$\open$, we need to compute the duplicity $\dups(s, U)$ using
equation~\eqref{eq:6} which requires constructing subtrees at $s$,
involving multiple expansions,  and doing the same for every
state $s' \in U_{R}$. These additional expansions can quickly become
significant especially in large state spaces and defeat the purpose of
using the subtree overlap metric to achieve fast planning times, as we
show in our experiments in Section~\ref{sec:experiments}.
To avoid this computational burden, we make an important observation that the subtree
overlap metric $\eta_{H}(s, s')$ is purely a function of the relative state
configuration of $s'$ with respect to $s$, and the robot dynamics. Thus, we can
precompute the subtree overlap metric for all possible relative state
configurations offline using the motion primitives, and store it for quick
lookup during search.

We implement this by finely discretizing the continuous state space
$\statespace$. It is important to note that the fine discretization
is only used for the purposes of precomputing and storing subtree overlap values, and
not used for search. We consider all possible discrete relative state configurations of
$s'$ with respect to any fixed state $s$ so that they still lie within a
euclidean distance of $R$ in the continuous space. For each such relative
configuration, we compute the subtree overlap metric $\eta_{H}(s, s')$ offline and store
it in a hash table. During search, for any pair of states $(s, s')$ we compute
the discretized relative configuration of $s'$ with respect to $s$ and query the
hash table to obtain the corresponding subtree overlap $\eta_{H}(s, s')$. The
resulting value is used in equation~\eqref{eq:5} within
Algorithm~\ref{alg:wastar}. Thus, the search does not construct subtrees or
compute overlap at runtime. This
allows the planner to avoid additional expansions due to subtree
construction, and
still retain the advantages of using
subtree overlap metric.

\section{Experiments and Results}
\label{sec:experiments}

In this section, we present our experimental results on two motion planning domains: a 3D
domain with a car-like robot that has differential constraints on the
turning radius, and a 5D domain with a unmanned aerial vehicle that has
constraints on linear acceleration and angular speed. Similar experimental domains have been
chosen in~\cite{DBLP:conf/iros/DuKSL19}. We compare our approach with
\weiapproach~\cite{DBLP:conf/iros/DuKSL19},
\rrt~\cite{DBLP:journals/ijrr/LaValleK01}, and
\wastar~\cite{DBLP:journals/ai/Pohl70}. \weiapproach{} is the state-of-the-art
approach for soft duplicate detection in continuous space search-based motion
planning while \wastar{} performs search in continuous space without any
duplicate detection. \rrt{} is a kinodynamic sampling-based motion planning algorithm that
directly operates in continuous state space. We have not compared against
hard duplicate detection
approaches~\cite{DBLP:journals/algorithmica/BarraquandL93,
  DBLP:conf/socs/GonzalezL11} 
as \weiapproach{} is already shown to outperform these approaches in the domains
we consider~\cite{DBLP:conf/iros/DuKSL19}. We present two versions of our
approach: \labeler{}, which does not precompute subtree overlap values
offline, and
\hashtable{} that precomputes and stores subtree overlap values in a hash table
as described in Section~\ref{sec:prec-subtr-overl}. Both versions use
Algorithm~\ref{alg:wastar} with the proposed duplicity function $\dups$ as
described in Section~\ref{sec:dupl-detect-using}. For all experiments,
we use an admissible heuristic $h(s)$ that is computed using BFS using only
position state variables on a discretized state space. All experiments
were run on an Intel i7-7500U CPU (2.7 GHz) with 8GB RAM. The source code for all experiments is open-sourced.\footnote{The code for all 3D experiments (including sensitivity analysis) can be found at \url{https://github.com/Nader-Merai/cspace3d_subtree_overlap}. The code for 5D experiments can be found at \url{https://github.com/Nader-Merai/cspace5d_subtree_overlap}}

\begin{table*}[t]
  \centering
  \scriptsize
  \begin{tabular}{|c|c|c|c|c||c|c|c|c|}
\hline
    & \multicolumn{4}{c||}{\scriptsize 3D} &
                                            \multicolumn{4}{c|}{\scriptsize
                                            5D} \\
    \hline
    & {\scriptsize \weiapproach{}}&{\scriptsize \hashtable{}}&
                                                           {\scriptsize
                                                           \labeler{}}&
                                                                        {\scriptsize
                                                                        \rrt{}} &
                                                                        {\scriptsize
                                                                        \weiapproach{}} &
                                                                  {\scriptsize
                                                                  \hashtable{}
                                                                  }&
                                                                       {\scriptsize
                                                                       \labeler{}
                                                                       }
                                                                                        &
                                                                                          {\scriptsize \rrt{}}
    \\
\hline
{\scriptsize Time (s)} & $11.5 \pm 1$ &$\mathbf{7.8 \pm 0.3}$ & $96
                                                                \pm
                                                                28$ &
                                                                      $2.1
                                                                      \pm
                                                                      0.4$ &
                                                                      $22
                                                                      \pm
                                                                      5$
                                                             &$\mathbf{9.5 \pm 1}$ &
                                                                     $84
                                                                     \pm
                                                                     6$
                                                                                &
                                                                                  $45
                                                                                  \pm
                                                                                  14$
    \\
\hline
{\scriptsize Cost} & $5.6 \pm 0.5$&$5.4 \pm 0.4$ & $5.3 \pm
                                                             0.4$ & $6
                                                                    \pm
                                                                    0.5$ &
                                                                    $3.3
                                                                    \pm
                                                                    1.7$&
                                                                    $1.3
                                                                    \pm
                                                                    0.1$
                                & $1.3 \pm 0.1$ & $1.8 \pm 0.2$ \\
    \hline
    {\scriptsize Expansions} & $57 \pm 14$&$8.2 \pm 1.4$ &
                                                                       $\mathbf{4.9
                                                                       \pm
                                                                       1}
                                                           (896 \pm 628)$
                                  & $-$ 
    & $46 \pm 18$ & $\mathbf{10.6 \pm 1.6}$ & $\mathbf{10.4 \pm 1.3} (875 \pm 110)$ & $-$\\
    \hline
  \end{tabular}
  \caption{Results for 3D and 5D experiments. Cost is reported as multiples of $10^4$ and
    expansions as multiples of $10^3$. \rrt{} does not perform
  expansions as it is a sampling-based planner. \wastar{} is not
  included as it was not successful in 
  computing a solution in $120$ seconds, and has a success rate of
  $40\%$ in 3D and $3\%$ in 5D. The mean and standard error statistics
are reported in the table. For \labeler{}, the expansions reported are
the number of expansions done during search and the parentheses contain
the total number of expansions including subtree construction.}
\label{tab:results}
\vspace{-0.2cm}
\end{table*}

\begin{figure}[t]
  \centering
  \begin{subfigure}{0.49\columnwidth}
    \includegraphics[width=\linewidth]{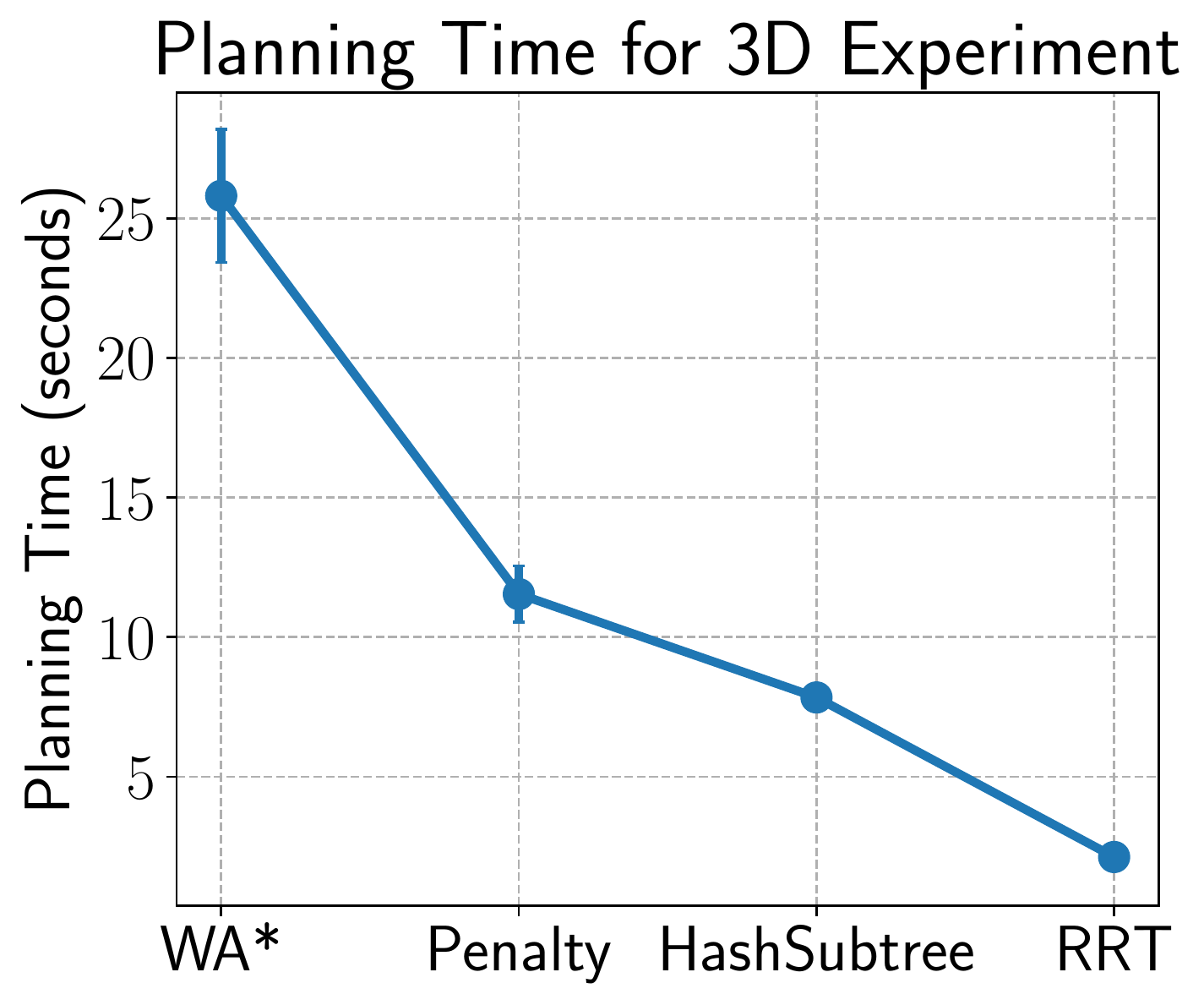}
  \end{subfigure}
  \begin{subfigure}{0.49\columnwidth}
    \includegraphics[width=\linewidth]{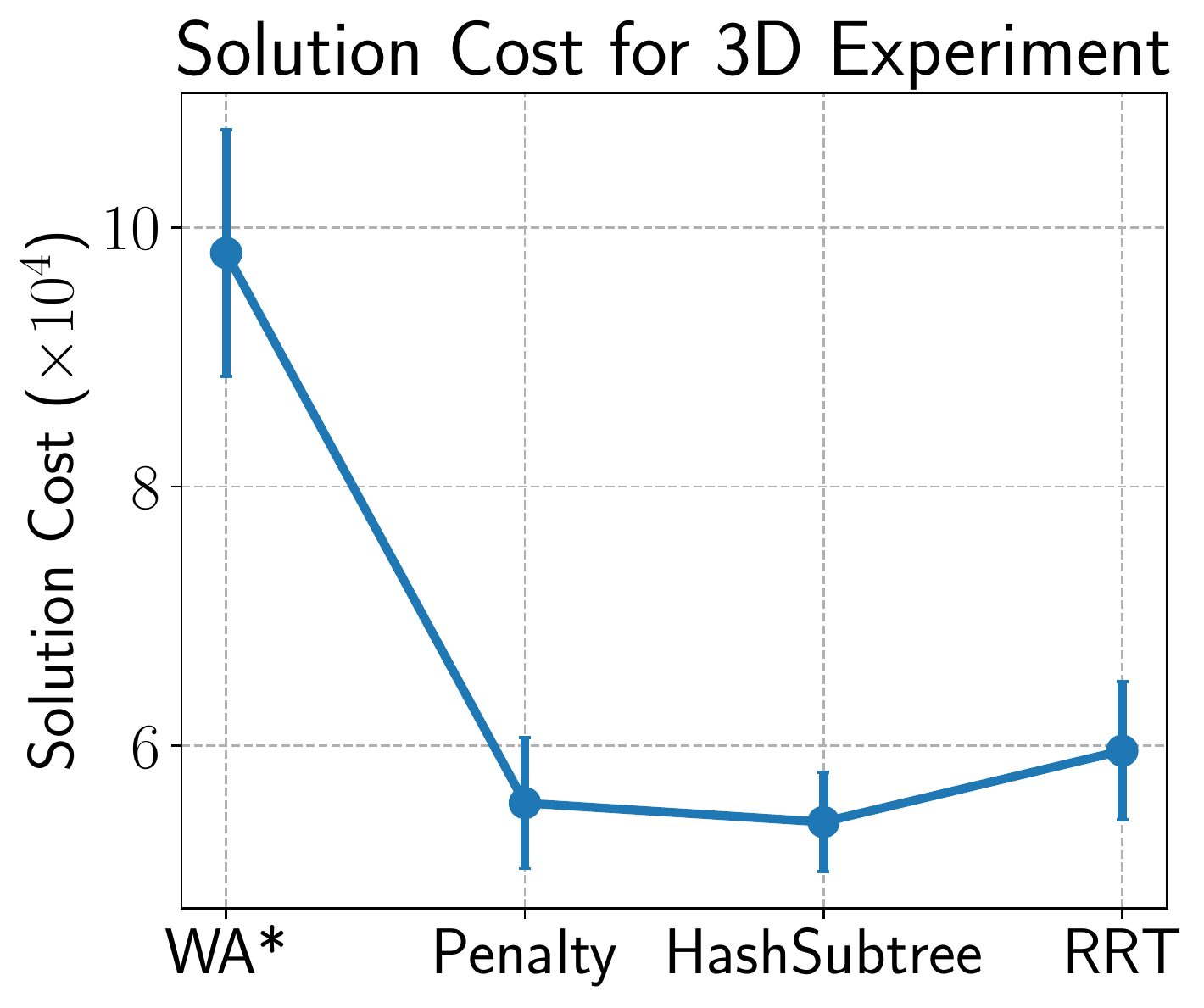}
  \end{subfigure}
  \caption{(left) Planning time and (right) Solution cost
    for 3D experiment. Error bars show standard error over $30$ runs.}
  \label{fig:3d-results}
  \vspace{-0.3cm}
\end{figure}
\begin{figure*}[t]
  \centering
  \begin{subfigure}{.15\linewidth}
    \includegraphics[width=\linewidth]{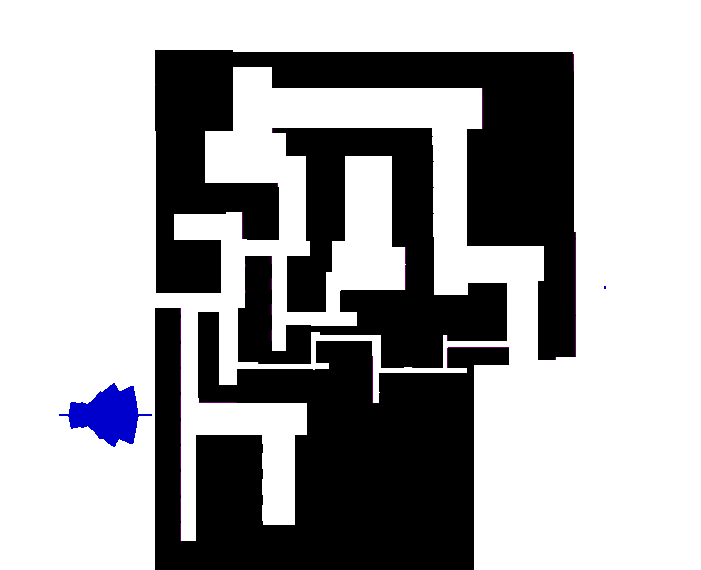}
    \caption{\wastar{}}
  \end{subfigure}\hfill
  \begin{subfigure}{.15\linewidth}
    \includegraphics[width=\linewidth]{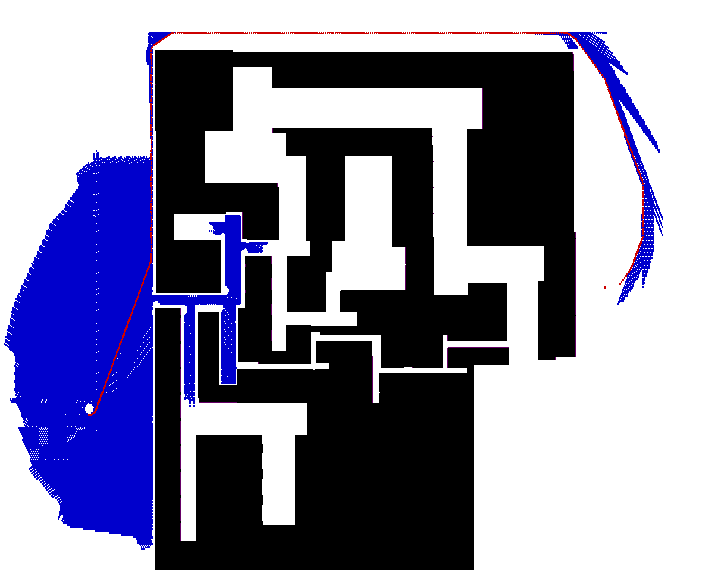}
    \caption{\weiapproach{}}
  \end{subfigure}\hfill
  \begin{subfigure}{.15\linewidth}
    \includegraphics[width=\linewidth]{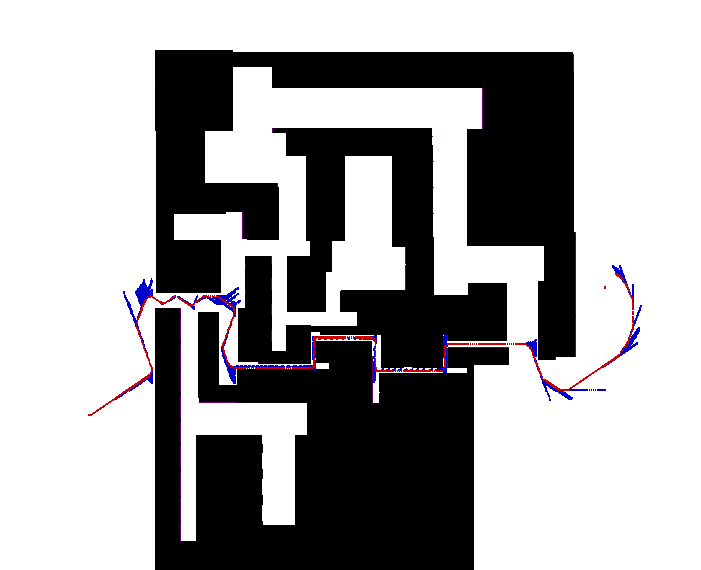}
    \caption{\labeler{}}
  \end{subfigure}\hfill
  \begin{subfigure}{.15\linewidth}
    \includegraphics[width=\linewidth]{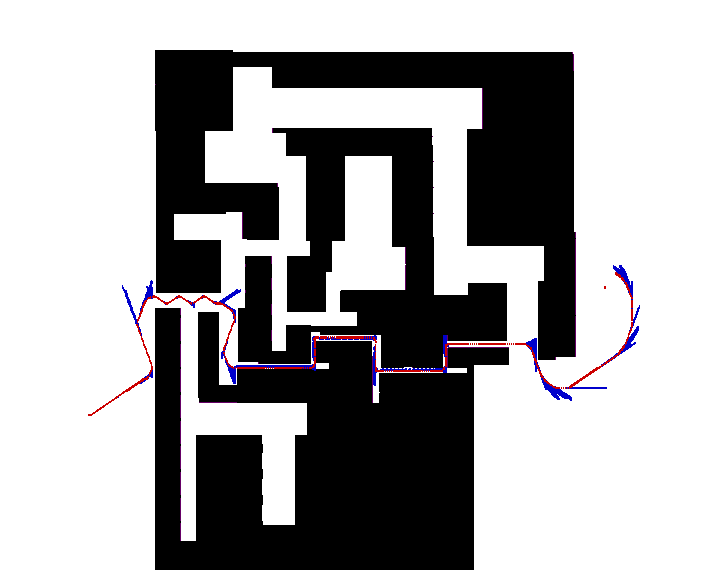}
    \caption{\hashtable{}}
  \end{subfigure}\hfill
  \begin{subfigure}{.15\linewidth}
    \includegraphics[width=\linewidth]{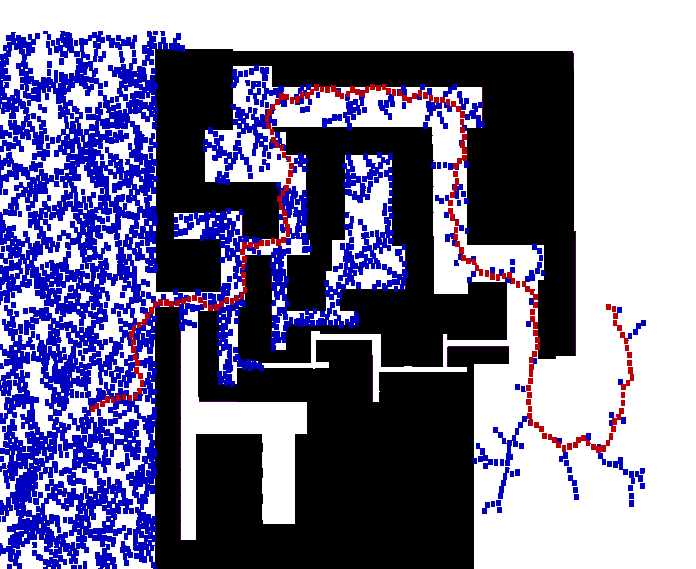}
    \caption{\rrt{}}
  \end{subfigure}
  \caption{Heatmaps showing results for 3D experiment. The red path
    shows the solution, and blue regions represent
    states that have been expanded. In the case of \rrt{}, blue states
    represent sampled states. The start is to the left of
    maze and the goal is to the right of maze. All approaches are
    successful, except for \wastar{} which
    could not find a solution in $120$ seconds.}
  \label{fig:heatmaps}
  \vspace{-0.2cm}
\end{figure*}
\begin{figure*}[t]
  \centering
  \begin{subfigure}{0.25\linewidth}
    \includegraphics[width=\linewidth]{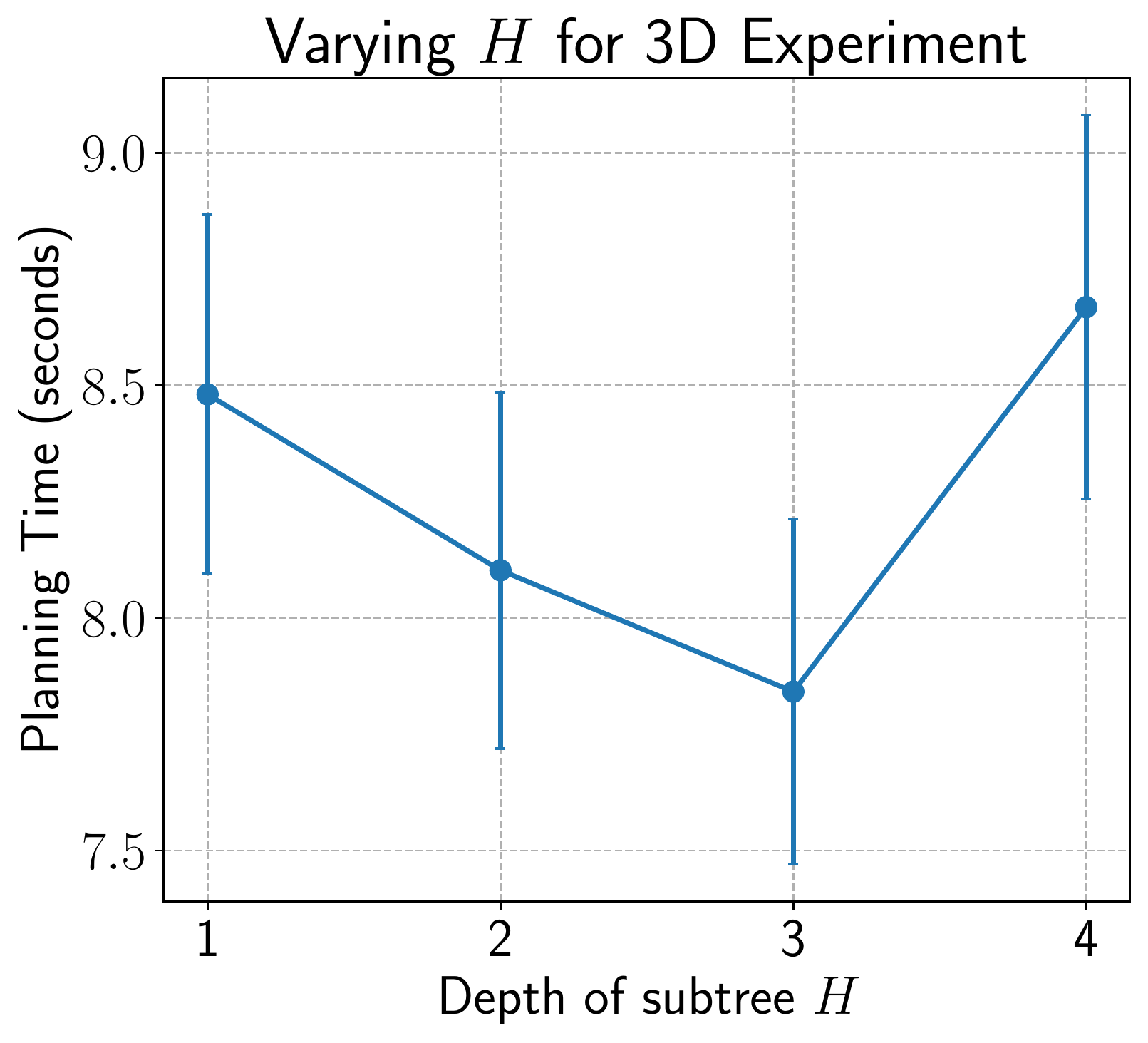}
  \end{subfigure}\hfill
  \begin{subfigure}{0.25\linewidth}
    \includegraphics[width=\linewidth]{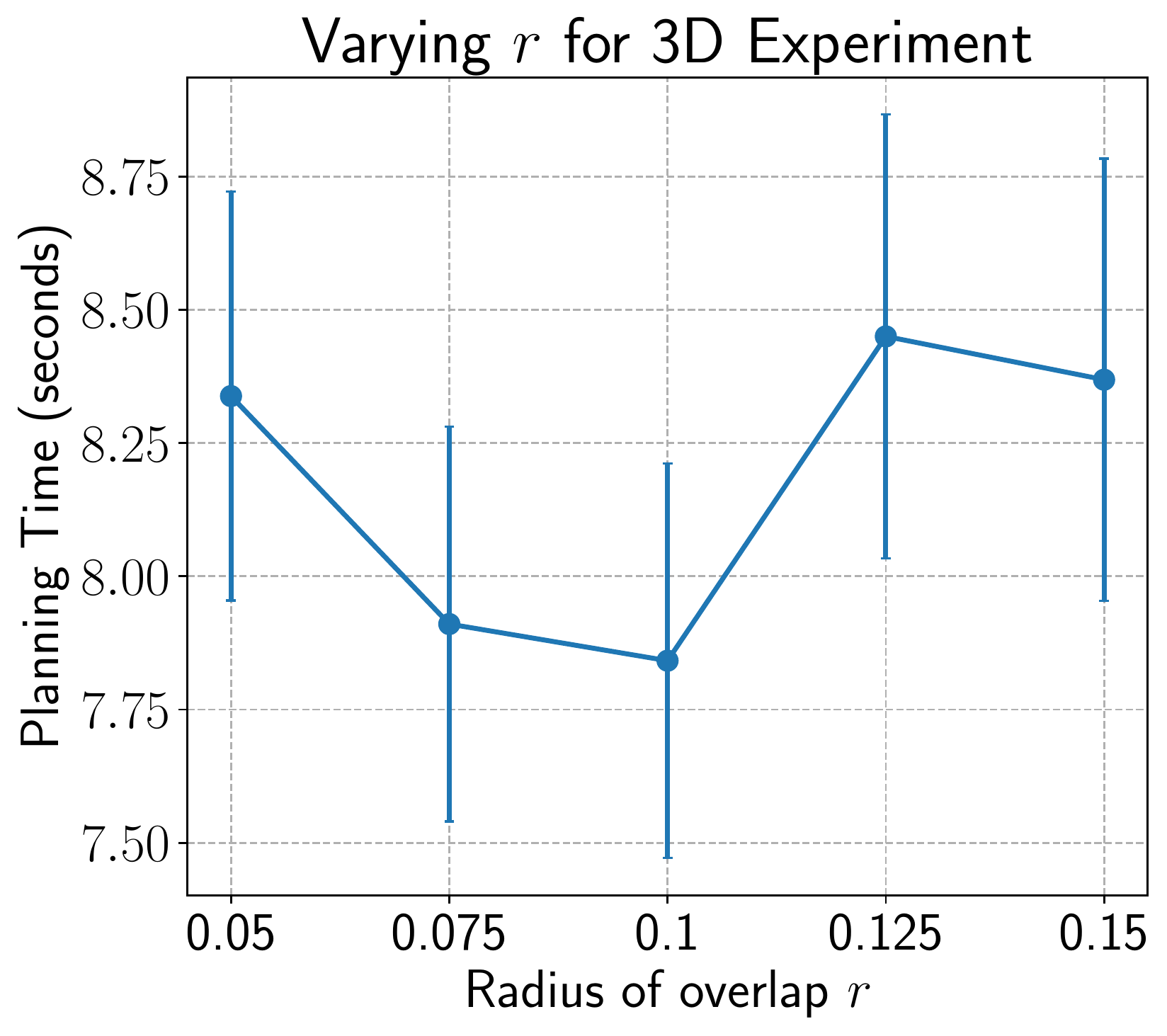}
  \end{subfigure}\hfill
  \begin{subfigure}{0.25\linewidth}
    \includegraphics[width=\linewidth]{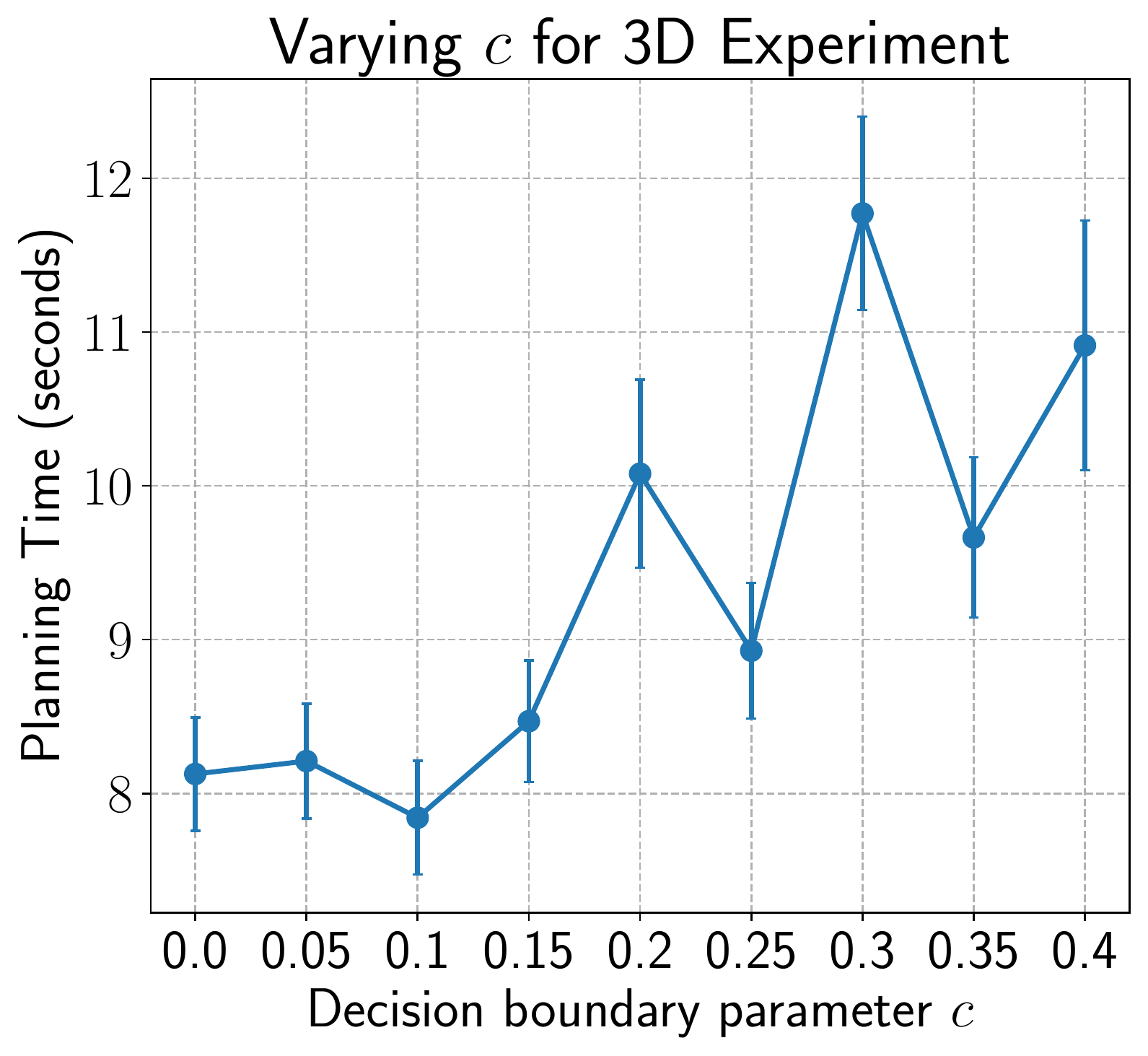}
  \end{subfigure}
  \caption{Planning time as the (left) subtree depth $H$, (middle)
    radius of overlap $r$, and (right) decision boundary parameter $c$
    are varied in 3D experiments while keeping the other parameters
    fixed. Error bars show standard error over $30$ runs.}
  \label{fig:sensitivity}
  \vspace{-0.2cm}
\end{figure*}

\subsection{3D Experiment}
\label{sec:3d-experiment}
Our first set of experiments involve a 3D planning domain with a
car-like robot. We use $3$
maps from Moving AI lab~\cite{DBLP:journals/tciaig/Sturtevant12} each with
$10$ randomly chosen start and goal pairs. Motion primitives are
generated for the robot using a unicycle model with constraints on the
turning radius resulting in $5$ successors for each expansion. The
state space is specified using $s = (x, y, \theta)$
where $(x, y)$ describe the position of the robot, and $\theta$
describes the heading. We
use values of $c$, $H$, and $r$ as the best performing values
according to experiments in Section~\ref{sec:sensitivity-analysis}. We
use $\epsilon_0 = 1$, and $\epsilonmax = 2$.

% The results are presented in Table~\ref{tab:results},
% Figure~\ref{fig:3d-results} and Figure~\ref{fig:heatmaps}. 
% \\\\Analyzing the results, we observe that the \textsc{HashSubtree}
% algorithm produces similar solution cost and expansions metrics to
% those of the \textsc{Subtree} algorithm, while having vastly lower
% planning time, as well as better solution cost and planning time
% metrics than the rest of the algorithms with only \textsc{RRT}
% edging it out on time, alongside a more robust-to-obstacles and
% narrow pathways expansion behaviour

The results are presented in Figure~\ref{fig:3d-results} and
Table~\ref{tab:results}. Figure~\ref{fig:3d-results} shows that our
approach \hashtable{} outperforms both \wastar{} and the
state-of-the-art soft duplicate detection 
\weiapproach{} in planning time by more than a factor of $1.5$, without
sacrificing solution cost (in fact, computes better solutions.) \rrt{} outperforms 
\hashtable{} in planning time but computes solutions with higher costs
which is typical of sampling-based planners that provide no guarantees
on sub-optimality of solution. The reason for the success of our
approach is illustrated in the example shown in
Figure~\ref{fig:heatmaps}. Our approaches \labeler{} and \hashtable{}
compute the solution with significantly less number of expansions when
compared to other approaches. Since expansions are the most expensive
operation in search, reducing the number of expansions leads to large
savings in planning time.

Table~\ref{tab:results} emphasizes the impact of using subtree overlap
metric in duplicity computation. We observe that \labeler{} that
constructs subtrees and computes subtree overlap during search
achieves least expansions during search, not counting the additional
expansions for subtree construction. However, as described in
Section~\ref{sec:prec-subtr-overl}, \labeler{} takes a long time to compute a
solution because of these additional expansions, and \hashtable{} avoids
this by precomputing subtree overlap 
values offline and using the stored values during search. It is important to note that
\hashtable{} incurs a small increase in number of expansions due to
discretization, but 
greatly reduces the planning time when compared to \labeler{}. Thus,
\hashtable{} retains the advantages of using subtree overlap metric
while achieving fast planning times.

\begin{figure}[t]
  \centering
  \begin{subfigure}{0.49\columnwidth}
    \includegraphics[width=\linewidth]{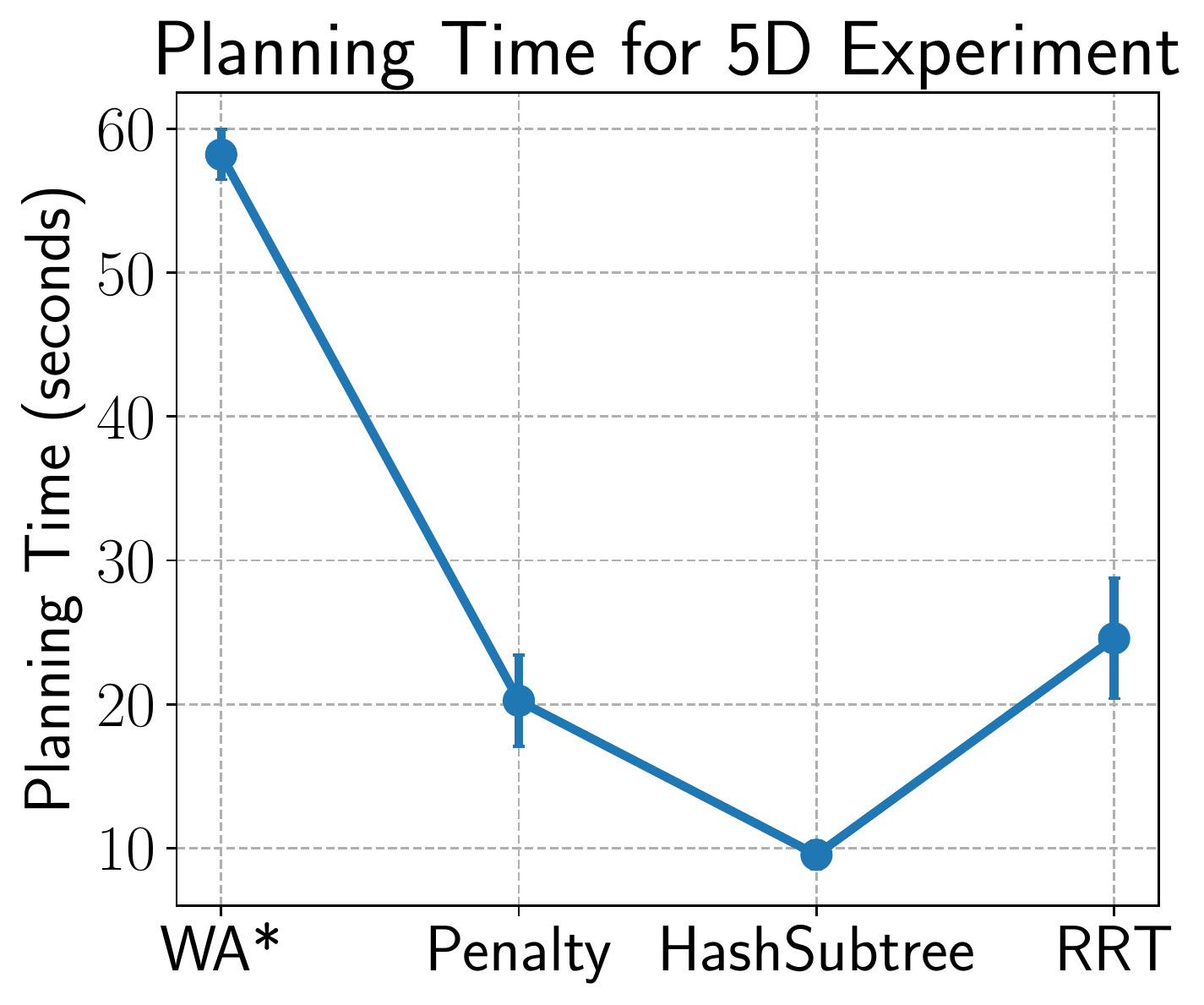}
  \end{subfigure}
  \begin{subfigure}{0.49\columnwidth}
    \includegraphics[width=\linewidth]{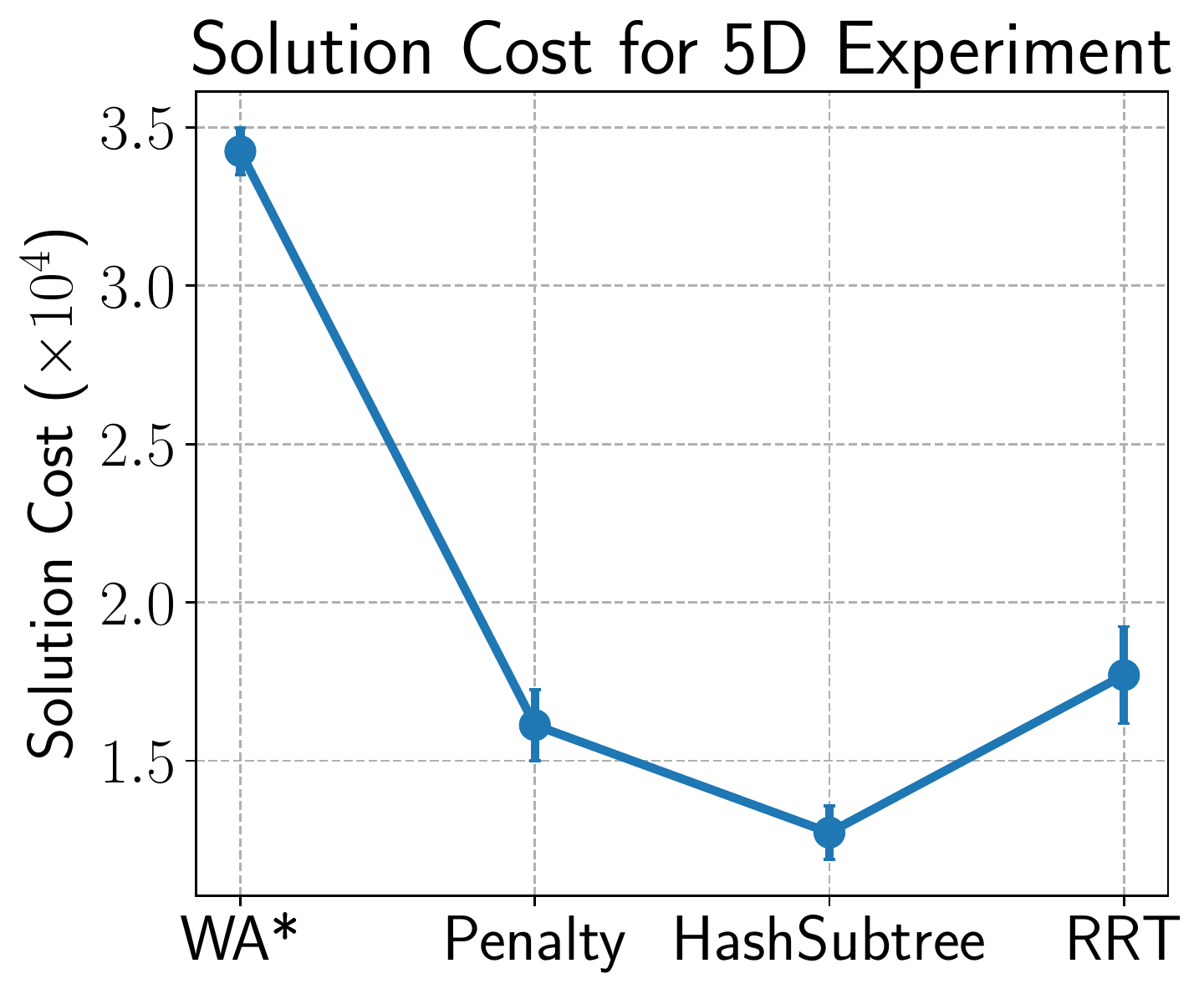}
  \end{subfigure}
  \caption{(left) Planning time and (right) Solution cost
    for 5D experiment. Error bars show standard error over $30$ runs.}
  \label{fig:5d-results}
  \vspace{-0.3cm}
\end{figure}

\subsection{5D Experiment}
\label{sec:5d-experiment}
Our second set of experiments involve a 5D planning domain with an
unmanned aerial vehicle that has constraints on linear acceleration
and angular speed. We use $5$
maps that are mesh models of real world with obstacles being no-fly
zones each with
$6$ randomly chosen start and goal pairs. Computing motion primitives
using two point boundary value problem solvers is difficult for this
domain. Thus, motion primitives are
generated for the robot using a local controller with inputs
planar acceleration in $XY$ plane $a_{xy}$, velocity in $Z$ axis $v_z$, and
the yaw angular speed $\omega$ resulting in $27$ successors for each
expansion. The state space is specified using $s = 
(x, y, z, \theta, v_{xy})$ 
where $(x, y, z)$ describe the position of the robot, $\theta$
describes the heading, and $v_{xy}$ the velocity in $XY$ plane. We use
$\epsilon_0 = 1$, $\epsilonmax = 5$, $c = 0.5$, $H = 1$, and $r = 0.1$.

The results are presented in Table~\ref{tab:results} and
Figure~\ref{fig:5d-results}. 
% \\\\Analyzing the results, we observe that the \textsc{HashSubtree} algorithm produces similar solution cost and expansions metrics to those of the \textsc{Subtree} algorithm, while having vastly lower planning time, as well as better solution cost and planning time metrics than the rest of the algorithms.
Figure~\ref{fig:5d-results} shows that \hashtable{} outperforms
other approaches in terms of planning time by more than a factor of $2$. The
heuristic computed using only $(x, y, z)$ locations is not as
informative in $5$D as it was in $3$D, which results in duplicate
detection playing a more important role. Consequently, \wastar{} is
only able to solve $1$ of the $30$ runs within $120$ seconds, while
\weiapproach{} solves all runs but takes twice as long planning times when
compared to \hashtable{}. \rrt{} also has large planning times as all
the maps have narrow passageways that the robot must pass through with
highly constrained dynamics, to get to the goal. Thus, \rrt{} uses a
lot of samples before it computes a solution leading to large planning
times. As expected, the cost of \rrt{} solution is also higher than
search-based approaches. Table~\ref{tab:results} shows that \labeler{}
achieves the least expansions during search but has large planning
times due to subtree construction, and
\hashtable{} achieves the best of both worlds with low expansions and
least planning time.

\subsection{Sensitivity Analysis}
\label{sec:sensitivity-analysis}
The final set of experiments involve the 3D planning domain described
before with the same maps, and start and goal pairs. These experiments aim to analyze how sensitive the performance of our
approach \hashtable{} is to the hyperparameters in the proposed duplicity
function $\dups$, i.e. depth of subtree $H$, radius of overlap $r$,
and decision boundary parameter $c$.

% The results are presented in Figure~\ref{fig:sensitivity}.
% \\\\Analyzing the results, we observe that the proposed approach is robust to a wide choice of hyper-parameters across $r$ and $H$, but not for $c$, which produced results of high variance, making it be the most crucial out of the three hyper-parameters, and requiring domain knowledge to tune and figure out properly.
The results are summarized in Figure~\ref{fig:sensitivity}. We observe
that the performance of \hashtable{} does not vary much with both
depth of the subtree $H$, and radius of overlap $r$. Thus, our
approach is robust to the choice of these hyperparameters. However,
the hyperparameter $c$ plays an important role in the performance of
\hashtable{} as shown in Figure~\ref{fig:sensitivity} (right) and
needs to be chosen carefully using prior domain knowledge to obtain
the best performance. More investigation is needed to understand the
best way to automatically chose this hyperparameter, and is left for
future work.

% \begin{figure}[t]
%   \centering
%   \begin{subfigure}{0.49\columnwidth}
%     \includegraphics[width=\linewidth]{Figures/depth_time_errbar.pdf}
%   \end{subfigure}
%   \begin{subfigure}{0.49\columnwidth}
%     \includegraphics[width=\linewidth]{Figures/radius_time_errbar.pdf}
%   \end{subfigure}
%   \caption{Planning time as the (left) subtree depth $H$ and (right)
%     radius of overlap $r$ are varied
%     for 3D experiment}
% \end{figure}

% \begin{figure}[t]
%   \centering
%   \includegraphics[width=\linewidth]{Figures/c_time_errbar.pdf}
%   \caption{Planning time as the decision boundary parameter $c$ is varied in 3D experiment}
% \end{figure}

\section{Conclusions and Future Work}
% We presented an approach to escaping local minimum regions for
% search-based planning in continuous spaces using offline
% computation. While the \textsc{HashSubtree} algorithm gave
% improvements in planning time, solution cost, and number of expansions
% over other baseline algorithms on 3D and 5D domains, it still suffers
% from the curse of dimensionality, since the hash table would grow
% exponentially with each additional dimension. One way of overcoming
% this problem would be to use a classifier instead of a hash table,
% which would be trained on a well-built data set of hash table keys as
% the data and hash table values as the labels, in order to imitate the
% behaviour of a hash table without requiring the huge amount of space
% it would require at higher dimensions.
We presented a soft duplicate detection approach for search-based
planning methods to operate in continuous spaces. Unlike previous
work~\cite{DBLP:conf/iros/DuKSL19} that uses a simple metric based on
euclidean distance to compute the duplicity of a state, we introduce a
more kinodynamically informed metric, namely subtree overlap, that encodes the
likelihood of a state contributing to the solution given that the
search has already seen other states.
% We motivate why our proposed
% metric is better at capturing duplicity of a state when compared to
% using euclidean distance alone using simple examples.
Theoretically,
we show that our approach has completeness and sub-optimality bound
guarantees using 
weighted A* and the proposed duplicity function. Empirically, our
approach outperforms previous approaches in terms of planning time by
a factor of $1.5$ to $2\times$ on $3$D and $5$D planning domains with
highly constrained dynamics. One direction for future work would be to
explore automatically selecting 
the decision boundary parameter $c$ given a new planning
domain. Furthermore, while \hashtable{} outperformed 
other approaches in terms of planning time in $3$D and $5$D domains,
it is infeasible to store subtree overlap values in a hashtable for
higher dimensional planning domains as the size of the hashtable grows
exponentially. A potential solution is to train a function
approximator offline to learn to predict subtree overlap $\eta_H(s, s')$ given
feature representations of the states $s, s'$. The trained function
approximator can be used during search to quickly compute the
duplicity of a state in high dimensional planning problems.

\section*{Acknowledgements}

AV would like to thank Muhammad Suhail Saleem and Rishi Veerapaneni for their help in reviewing the paper. The authors would like to thank Wei Du for sharing the code to setup experiments, and Nikita Rupani for her initial work that led to some of the insights in this work.

\bibliographystyle{IEEEtran}
\bibliography{ref}

\end{document}